\documentclass[letterpaper, 10 pt, conference]{ieeeconf}  

\IEEEoverridecommandlockouts                              

\usepackage{graphics} 
\usepackage{epsfig} 
\usepackage{mathptmx} 
\usepackage{times} 
\usepackage{amsmath} 
\usepackage{amssymb}  
\usepackage{multirow}
\usepackage{booktabs}
\usepackage{ulem} 
\usepackage{cite}
\usepackage{stfloats}
\usepackage{xcolor}
\usepackage[caption=false]{subfig}
\usepackage{algorithm}
\usepackage[noend]{algpseudocode}
\newcommand{\warning}[1]{\textcolor[RGB]{0, 0, 0}{#1}}

\title{\LARGE \bf
A$^2$I-Calib: An Anti-noise Active Multi-IMU Spatial-temporal Calibration Framework for Legged Robots
}

\author{Chaoran Xiong$^{\dag}$, ~\IEEEmembership{Student~Member,~IEEE}, Fangyu Jiang$^{\dag}$, Kehui Ma, Zhen Sun, \\ Zeyu Zhang, and Ling Pei$^{\ast}$, ~\IEEEmembership{Senior~Member,~IEEE} 
\thanks{$^{\dag}$Chaoran Xiong and Fangyu Jiang contribute equally to this work.}%
\thanks{$^{\ast}$Corresponding author: Ling Pei.}
\thanks{This work was supported in part by National Nature Science Foundation of China (NSFC) (Grant Number: 62273229), and in part by Science and Technology Commission of Shanghai Municipality (Grant Number: 24TS1402600 and 24TS1402800).}
\thanks{The authors are with the Shanghai Jiao Tong University, Shanghai 200240,
China (e-mail: sjtu4742986; jiangfangyu; khma0929; zhensun; zhang-zeyu; ling.pei@sjtu.edu.cn)}%
\thanks{The code will be released at {https://github.com/DavidGrayrat/A2I-Calib}.}%
}

\newtheorem{theorem}{Theorem}  

\begin{document}

\maketitle
\thispagestyle{empty}
\pagestyle{empty}

\begin{abstract}

Recently, multi-node inertial measurement unit (IMU)-based odometry for legged robots has gained attention due to its cost-effectiveness, power efficiency, and high accuracy. However, the spatial and temporal misalignment between foot-end motion derived from forward kinematics and foot IMU measurements can introduce inconsistent constraints, resulting in odometry drift. Therefore, accurate spatial-temporal calibration is crucial for the multi-IMU systems. Although existing multi-IMU calibration methods have addressed passive single-rigid-body sensor calibration, they are inadequate for legged systems. This is due to the insufficient excitation from traditional gaits for calibration, and enlarged sensitivity to IMU noise during kinematic chain transformations. To address these challenges, we propose A$^2$I-Calib, an anti-noise active multi-IMU calibration framework enabling autonomous spatial-temporal calibration for arbitrary foot-mounted IMUs. Our A$^2$I-Calib includes: 1) an anti-noise trajectory generator leveraging a proposed basis function selection theorem to minimize the condition number in correlation analysis, thus reducing noise sensitivity, and 2) a reinforcement learning (RL)-based controller that ensures robust execution of calibration motions. Furthermore, A$^2$I-Calib is validated on simulation and real-world quadruped robot platforms with various multi-IMU settings, which demonstrates a significant reduction in noise sensitivity and calibration errors, thereby improving the overall multi-IMU odometry performance. 

\end{abstract}

\section{INTRODUCTION}
Real-time, low-power, and high-precision state estimation is the foundation for the locomotion control, velocity tracking\cite{MIT}, and path planning\cite{Planning} for legged robots in the embodied navigation systems. Recently, multi-IMU-based odometry has emerged as a promising solution due to its low cost and energy efficiency\cite{MIMC-VINS, MIPH, MIPO}. 

Typically, the multi-IMU systems for legged robots leverage the foot IMU measurements to augment base state estimation through kinematic constraints\cite{MIPO, MIPH}. This approach enhances estimation accuracy by fusing proprioceptive measurements from multiple foot-end nodes with traditional body-IMU propagation.
However, the spatial-temporal misalignment between IMU mounting positions and forward kinematics-derived foot-end motions introduces inconsistent measurement constraints, thus leading to significant odometry drift. Therefore, accurate spatial-temporal calibration is crucial for the multi-IMU systems.
\begin{figure}
    \centering
    \includegraphics[width=1.0\linewidth]{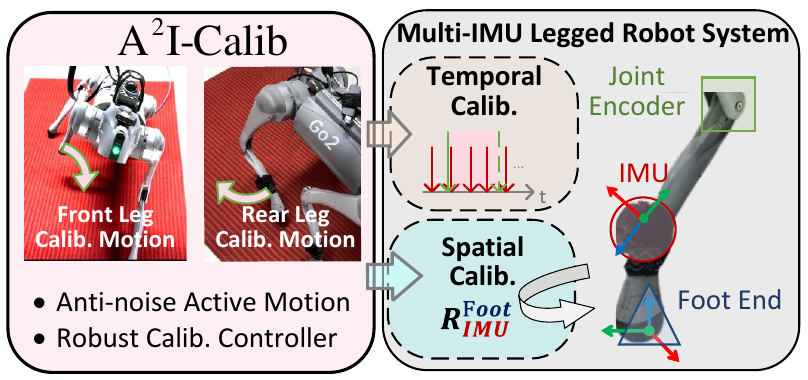}
    \caption{A$^2$I-Calib is a novel spatial-temporal calibration framework for arbitrary foot-mounted IMUs. In this framework, a basis function selection theorem is proposed to minimize the condition number in calibration correlation analysis, thereby reducing noise sensitivity. Moreover, a RL-based motion controller is designed to implement the optimal calibration trajectory. Finally, more precise calibration results from A$^2$I-Calib improve the overall multi-IMU odometry performance.}
    \label{fig:abstract}
    \vspace{-0.5cm}
\end{figure}
Existing multi-IMU extrinsic calibration methods have addressed single-rigid-body sensor calibration\cite{qiu2021, iKalibr, EC-MI}. These solutions employ external sensors such as vision\cite{OMIC} and LiDAR\cite{ML-MI}, or temporal correlation analysis\cite{qiu2021} for multi-IMU calibration. Though effective on single-rigid-body multi-IMU platforms such as drones, they are inadequate for legged systems due to two critical limitations: 1) insufficient excitation from conventional gaits for calibration algorithms, and 2) enlarged sensitivity to IMU noise during kinematic chain transformations. Moreover, these calibration frameworks all require manual intervention and demonstrate suboptimal performance under the typical locomotions of the legged robots.

To address these challenges, we present A$^2$I-Calib, an active anti-noise calibration framework that enables fully autonomous spatial-temporal calibration of arbitrary foot-mounted IMUs, as shown in Fig. \ref{fig:abstract}. Our A$^2$I-Calib framework includes: 1) A trajectory generation module employing a proposed basis function selection theorem to minimize condition numbers in correlation analysis, thereby suppressing noise sensitivity; and 2) A reinforcement learning (RL)-based calibration action controller ensuring robust execution of calibration-specific motions. Comprehensive validations on simulation and real-world platforms demonstrate the capability of A$^2$I-Calib to achieve high-precision spatial-temporal calibration. Our main contributions are as follows:
\begin{enumerate}
    \item An anti-noise active multi-IMU calibration framework, namely A$^2$I-Calib, which enables autonomous spatial-temporal calibration for arbitrary foot-mounted IMUs.
    \item A basis function selection theorem for the anti-noise trajectory generation algorithm to minimize the \warning{condition number, thereby reducing} noise sensitivity. To execute the designed optimal trajectory, a RL-based robust calibration action controller is introduced with rewards designed for tripod balancing and stable calibration action implementation. 
    \item Comprehensive validations on the simulation and real-world quadruped robot platforms with various multi-IMU settings, demonstrating a significant reduction in error sensitivity and calibration errors, and an enhancement in multi-IMU odometry accuracy. 
\end{enumerate}

The remainder of this paper is organized as follows: The related work on multi-IMU calibration algorithms is reviewed in Section II. Section III describes the formulation of spatial-temporal misalignment and noise sensitivity for foot IMU calibration. Section IV introduces the two key modules designed for A$^2$I-Calib. Section V presents a comparison of A$^2$I-Calib with other calibration motions on simulation and real-world quadruped robot platforms with various IMU settings, including calibration accuracy and its influence on multi-IMU odometry performance. Finally, the conclusion is given in Section VI .

\section{RELATED WORK}
In legged robot platforms, multi-IMU systems are commonly used for real-time state estimation\cite{MIPH, MIPO}. In such systems, the body IMU propagates the base state, while the foot IMUs, combined with forward kinematics, conduct the error correction. To temporally and spatially align the foot IMU data with forward kinematics, it is necessary to calibrate the rotational extrinsic parameters and time offset between the foot IMU and the foot-end motion computed by the joint encoders through forward kinematics.

Existing multi-IMU calibration algorithms can be categorized into external sensor-assisted methods and IMU-only trajectory correlation-based calibration methods. As for external sensor-assisted methods, Rehder et al.\cite{E-Kalibr} first extended the well-known calibration toolbox Kalibr to support single camera-aided multi-IMU extrinsic calibration. This method requires the visual chessboards for calibration objective function construction. Similarly, Li et al.\cite{ML-MI} employed LiDAR to assist multi-IMU calibration. In addition, an online calibration method with visual-inertial odometry was proposed in \cite{OMIC}. On the other hand,  IMU-only trajectory correlation-based calibration methods were presented in \cite{qiu2021,EC-MI}. These algorithms were based on the correlation and covariance between IMU trajectories. Though the trajectory correlation-based methods do not rely on external sensors, the calibration accuracy depends heavily on the motion trajectories.

To date, current multi-IMU calibration algorithms are based on passive data collection from a single rigid body. However, for legged robot multi-IMU systems, ideal constraints of a single rigid body need to be substituted with forward kinematics, which amplifies the inherent Gaussian white noise of the IMUs. This process significantly affects the accuracy of the final spatial-temporal calibration. The existing gaits of legged robots\cite{Gait1,Gait2} are insufficiently excited for the calibration purpose and the motion trajectories enlarges the amplification of IMU noise. As a result, passive data collection methods based on existing gaits perform poorly in multi-IMU calibration for legged robots. Therefore, a new anti-noise active legged foot IMU calibration algorithm is needed to improve calibration accuracy and robustness.

\section{PROBLEM FORMULATION}

In this section, an overview of the foot IMU calibration problem and essential mathematical theories are provided. Section II-A formulates the problems of IMU foot-end extrinsic rotation and IMU-joint encoder time offset calibration. Then the noise sensitivity description for foot IMU calibration is presented in Section II-B.

\subsection{IMU Foot-end Spatial Temporal Calibration}
In the multi-IMU systems for the legged robots, the foot IMU is typically used to observe the state of the robot's foot end, such as the rolling velocity and slip detection when robot's paws touch the ground. This information is then used to correct the robot's base state. To obtain the state of the foot end, it is essential to calibrate the rotation matrix between the foot IMU and the foot-end coordinate system derived from forward kinematics, as defined by
\begin{equation}
    \omega^{F}_n = \mathbf{R}_I^F {\omega}^I_n,
\end{equation}
where $\omega^{F}_n$ and ${\omega}^I_n$ are the foot-end angular velocity derived from forward kinematics and the foot IMU angular velocity at time step $n$, respectively. $\mathbf{R}_I^F$ is the extrinsic rotation matrix for the spatial calibration.

Furthermore, the base state correction of a legged robot necessitates foot IMU measurements to be combined with joint encoders through forward kinematics. However, even with the same clock source, due to information transmission delays and sampling errors, there inevitably exists a time offset between the IMU and the joint encoder timestamps, which is given by
\begin{equation}
t_d = t_I - t_E,
\end{equation}
where $t_I$ and $t_E$ are the received sampling timestamps of the foot IMU and the joint encoder. $t_d$ is the time offset between them. This time offset can lead to errors in the body state correction by the foot IMU. Therefore it needs to be calibrated in the multi-IMU system.


\begin{figure*}
    \centering
    \includegraphics[width=0.9\linewidth]{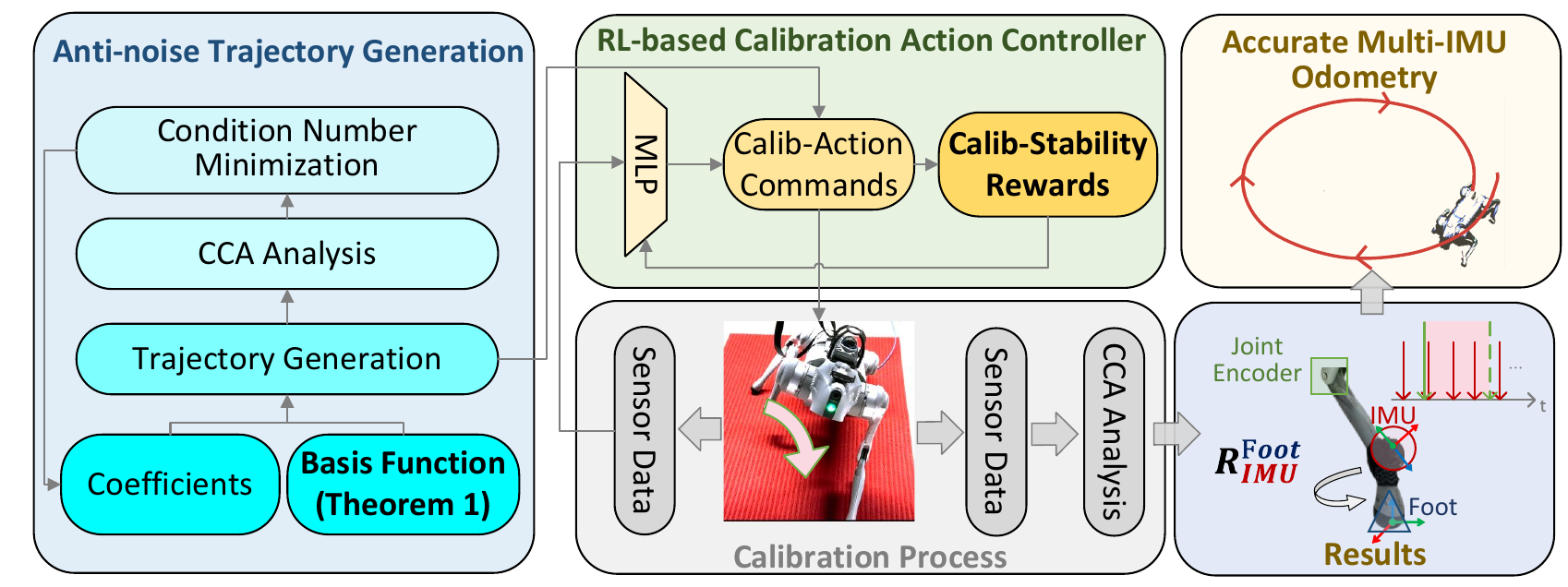}
    \caption{System overview of A$^2$I-Calib. Firstly, the anti-noise trajectory generation module generates and optimizes leg trajectories that minimizes the condition number in the legged robot, based on the proposed basis functions. Secondly, In order to perform the ideally generated calibration actions on the ground for the legged robot, the RL-based calibration action controller is introduced. This module achieves robust execution of anti-noise calibration actions. The rewards for lifting a single leg and combining it with the flexibility of the calibration actions are adopted in the RL training. Thirdly, the generated calibration commands are implemented into a quadruped robot and canonical correlation analysis (CCA) is conducted to calibrate the foot IMU. Finally, The calibration results, including the external rotation matrix and time offset, are then input into the multi-IMU odometry.}
    \label{fig:framework}
\end{figure*}
\subsection{Noise Sensitivity in Foot IMU Calibration}
In multi-IMU legged robot systems, foot IMU nodes typically use low-cost IMUs with relatively large noise. However, during the calibration process, the inherent Gaussian white noise present in the foot IMU can affect the calibration accuracy. The spatial-temporal calibration process between multiple nodes generally uses the CCA framework proposed in \cite{qiu2021}. In legged robots, this process involves performing a correlation analysis between the foot IMU's angular velocity sequence shifted by $t_d$ and the theoretical foot-end angular velocity calculated using forward kinematics from the joint encoders. The auto-covariance matrix and cross-covariance matrix of these two sequences are given by

\begin{align}
    ^{t_d}\Sigma_{II} & \approx \frac{1}{N-1} \sum_{n=1}^N\left(^{t_d}\omega^I_n-\overline{\boldsymbol{^{t_d}\omega^I}}\right)\left(\boldsymbol{^{t_d}\omega}^I_n-\overline{\boldsymbol{^{t_d}\omega}^I}\right)^T, 
    \\
    \Sigma_{FF}  & \approx \frac{1}{N-1} \sum_{n=1}^N\left(\omega^F_n-
\overline{\boldsymbol{\omega^F}}\right)\left(\boldsymbol{\omega}^F_n-\overline{\boldsymbol{\omega}^F}\right)^T, \label{Eq:FF}
    \\
    ^{t_d}\Sigma_{IF} & \approx \frac{1}{N-1} \sum_{n=1}^N\left(^{t_d}\omega^I_n-\overline{\boldsymbol{^{t_d}\omega^I}}\right)\left(\boldsymbol{\omega}^F_n-\overline{\boldsymbol{\omega}^F}\right)^T,
    \\
    ^{t_d}\Sigma_{FI} &= ^{t_d}\Sigma^T_{IF},
\end{align}
where $^{t_d}\omega^I_n$ is the foot IMU angular velocity shifted by $t_d$. $^{t_d}\Sigma_{II}, \Sigma_{FF}, ^{t_d}\Sigma_{FI}$ and $^{t_d}\Sigma_{IF}$ are the auto-covariance matrix of foot IMU angular velocity shifted by $t_d$, auto-covariance matrix of theoretical foot-end angular velocity, and cross-covariance matrixes of these two angular velocity sequences shifted by $t_d$, respectively. Then, one can deduce the trace correlation coefficient $r(^{t_d}\omega^I_n,\omega^F_n)$ of these two sequences by
\begin{equation}
    r(^{t_d}\omega^I_n,\omega^F_n)=\sqrt{\frac{1}{3} \operatorname{Tr}\left( {^{t_d}\Sigma_{II}^{-1}}  {^{t_d}\Sigma_{IF}} \Sigma_{FF}^{-1} {^{t_d}\Sigma_{FI}}\right)}.
\end{equation}
Details may refer to \cite{qiu2021}.
Furthermore, the time offset can be calibrated by the maximization of $r(^{t_d}\omega^I_n,\omega^F_n)$, as expressed by 
\begin{equation}
    t_d = \arg \max _{t_{d} \in \tau}  r(^{t_d}\omega^I_n,\omega^F_n).
\end{equation}

After the temporal calibration, the external rotation matrix $\mathbf{R}_I^F$ can be derived from the singular value decomposition of $\Sigma_{FF}^{-1}{ ^{t_d}\Sigma_{FI}}$, which is given as
\begin{equation}
    \Sigma_{FF}^{-1}{ ^{t_d}\Sigma_{FI}} = U \Sigma V^T,
\end{equation}
\begin{equation}
    \mathbf{R}_I^F=\left(\boldsymbol{U}\left(\begin{array}{ccc}
1 & 0 & 0 \\
0 & 1 & 0 \\
0 & 0 & \operatorname{det}\left(\boldsymbol{U} \boldsymbol{V}^T\right)
\end{array}\right) \boldsymbol{V}^T\right)^{-1}.
\end{equation}

Eq. (9) yields that the condition number of $\Sigma_{FF}$ determines the ill-conditioned level of this calibration process, see \cite{golub2013, CN}.
Therefore, the noise sensitivity is described by the condition number of $\Sigma_{FF}$ in this paper. The goal of the proposed A$^2$I-Calib is to minimize this condition number, as formulated by
\begin{equation}
    \{\mathbf{a}_0,\mathbf{a}_1, ..., \mathbf{a}_N\} = \arg \min _{\mathbf{a}_{t} \in \mathcal{A} }  \kappa(\Sigma_{FF}),
\end{equation}
where $\{\mathbf{a}_0,\mathbf{a}_1, ..., \mathbf{a}_N\}$ are active calibration action sequences. $\kappa(\Sigma_{FF})$ denotes the condition number of $\Sigma_{FF}$. $\mathcal{A}$ is the action space of the legged robot.

\section{METHODOLOGY}
In this section, we first introduce A$^2$I-Calib, an anti-noise active multi-IMU spatial-temporal calibration framework for legged robots in multi-IMU systems. Then the anti-noise calibration motion trajectory optimization algorithm and RL-based calibration action controller are presented. Finally, the implementation settings of A$^2$I-Calib on quadruped robots are provided.

\subsection{System Overview}
A$^2$I-Calib, our anti-noise active multi-IMU spatial-temporal calibration framework is illustrated in Fig. \ref{fig:framework}. Firstly, to reduce the noise sensitivity during the calibration process of the legged robots, the anti-noise trajectory generation module is designed. This module generates and optimizes motion trajectories that minimize the condition number of the foot-end auto-covariance matrix $\kappa(\Sigma_{FF})$, based on the proposed basis function selection theorem. In order to perform the ideally generated calibration actions on the ground for the legged robot, the RL-based calibration action controller is introduced. This module achieves stable execution of anti-noise calibration actions autonomously. The rewards for lifting a single leg and combining it with the flexibility of the calibration actions are designed in the RL training. Next, the generated calibration motion commands are implemented into a quadruped robot and the CCA process is conducted to calibrate the foot IMU. The calibration results, including the external rotation matrix $\mathbf{R}_I^F$ and time offset $t_d$, are then input into the multi-IMU odometry. Finally, the accuracy of the odometry is used to verify the effectiveness of the calibration.
\begin{algorithm}[t]
\caption{Anti-noise Referenced Trajectory Generation}
\label{alg:acmp}
\begin{algorithmic}[1]
\Require 
\Statex IMU frequency $f_{\mathrm{IMU}}$
\Statex Sensor time offset range $t_r$
\Ensure 
\Statex Joint angle trajectories ${\theta}^{\text{hip}}_n$, ${\theta}^{\text{th}}_n$, ${\theta}^{\text{calf}}_n$

\State Initialize:
\State $f \gets \pi/(4t_r)$ 
\State $T \gets 8t_r$ 
\State $\lambda \gets \{i/f_{\mathrm{IMU}} \mid i = 0,1,\ldots,T \cdot f_{\mathrm{IMU}}\}$

\For{$i \leq i_{\max}$}
    \State Compute joint angles:
    \State $\dot{\theta}^{\text{hip}}_n \gets \sum_{k=1}^{N} A_k \sin(kfn), \quad \forall n \in \lambda$
    
    \State Compute angular velocities:
    \State $\dot{\theta}^{\text{th}}_n + \dot{\theta}^{\text{calf}}_n \gets \sum_{k=1}^{N} B_k \cos(kfn), \quad \forall n \in \lambda$
    
    \State Integrate joint motions:
    \State ${\theta}^{\text{idx}}_n = \Sigma \dot{\theta}^{\text{idx}}_n, \quad \text{idx} = \text{hip},\text{th},\text{calf}$
    
    \State Compute foot end angular velocity:
    \State $\omega^{F}_n \gets f_{\mathrm{FK}}({\theta}_n, \dot{\theta}_n)$
    
    \If{$\kappa(\Sigma_{\mathrm{FF}}) < \kappa_{\mathrm{obj}}$ \textbf{and} $\theta^\text{idx}_n \in [\theta^{\text{idx}}_b, \theta^{\text{idx}}_u]$}
        \State \Return $\theta^{\text{hip}}_n, \theta^{\text{th}}_n, \theta^{\text{calf}}_n$
    \Else
        \State Compute motion ranges:
        \State $\Delta\theta^\text{idx} \gets \max(\theta^\text{idx}_n) - \min(\theta^\text{idx}_n),$
        
        \State Evaluate constraints:
        \State $w^\text{idx} \gets \begin{cases}
            0 & \text{if } \theta^\text{idx}_n \in [\theta_{b}^\text{idx}, \theta_u^\text{idx}] \\
            1 & \text{otherwise}
        \end{cases}$
        
        \State Compute loss and gradients:
        \State $\mathcal{L}, \nabla_{A_k}, \nabla_{B_k} \gets \mathrm{Loss}(\mathbf{A},\mathbf{B})$
        
        \State Update parameters:
        \State $A_k \gets A_k + \alpha \nabla_{A_k}$
        \State $B_k \gets B_k + \alpha \nabla_{B_k}$ 
    \EndIf
\EndFor
\State \Return $\theta^\text{hip}_n, \theta^\text{th}_n, \theta^\text{calf}_n$
\end{algorithmic}
\end{algorithm}

\subsection{Anti-noise Trajectory Generation}
Unlike traditional single-rigid-body multi-node sensor systems, the foot-end nodes of legged robots are not calibrated through direct sensor odometry observations. Instead, legged robots need to deduce the angular velocity sequence of the foot-end using forward kinematics with joint encoders. The regular foot-end motion of the quadruped robot is often insufficiently excited, which results in a large condition number, thereby enlarging the noise sensitivity for calibration. To tackle these issues, this anti-noise trajectory generation module optimizes the foot-end calibration motion based on the proposed basis function selection theorem, minimizing the condition number to its theoretical minimum. Therefore, it can significantly reduce the calibration algorithm's sensitivity to IMU noise. The detailed description of the trajectory generation algorithm design is given as follows.

A typical leg of a legged robot has three joints: the hip joint, the thigh joint, and the calf joint. Taking Unitree Go2 quadruped robot as a typical example, the link twist angles of one leg are defined by modified Denavit-Hartenberg parameters as ${\alpha}^{hip}={\alpha}^{calf}={\alpha}^{foot}=0, {\alpha}^{th}=-90^{\circ}$.
By controlling the angular velocity of each joint, regardless of the floating base, one can deduce the foot-end angular velocities by forward kinematics\cite{FK}, as given by  
\begin{equation}
\omega_{n}^F=\left(\begin{array}{c}
-\dot{\theta}_n^{\text{hip}} \cdot \sin \left(\theta^{\text{th}}_{n}+\theta^{\text{calf}}_{n}\right) \\
-\dot{\theta}_n^{\text{hip}} \cdot \cos \left(\theta^{\text{th}}_{n}+\theta^{\text{calf}}_{n}\right) \\
\dot{\theta}^{\text{th}}_{n}+\dot{\theta}^{\text{calf}}_{n}
\end{array}\right) \triangleq \left(\begin{array}{c}
\omega^x_{n} \\
\omega^y_{n} \\
\omega^z_{n}
\end{array}\right),
\label{Eq:FK}
\end{equation}
where $\dot{\theta}_n^{\text{hip}}$, $\dot{\theta}_n^{\text{th}}$ and $\dot{\theta}_n^{\text{calf}}$ are the angular velocities control commands for the hip joint, thigh joint, and calf joint, respectively. $\omega^x_{n}$, $\omega^y_{n}$ and $\omega^z_{n}$ are the angular velocity along with x, y and z axis, respectively.

In order to generate the optimal foot end trajectory to minimize the condition number of $\Sigma_{FF}$, theorem \ref{th:1} for basis function selection is proposed to simplify the problem.
\begin{theorem}
During an integer number of calibration periods $(0, NT)$, the auto-covariance matrix of theoretical foot-end angular velocity $\Sigma_{FF}$ is a diagonal matrix so that the condition number is significantly reduced\cite{CN} when the hip joint angular velocity $\dot{\theta}^{\text{hip}}_n$ is a linear combination of functions from the sine function set $\boldsymbol{S} = \{\sin(k f n)\}$, and the sum of the thigh joint and calf joint angular velocity $\dot{\theta}^{\text{th}}_n + \dot{\theta}^{\text{calf}}_n$ is a linear combination of functions from the cosine function set $\boldsymbol{C} = \{\cos(k f n)\}$; that is
\begin{align}
\dot{\theta}^{\text{hip}}_n &= \sum_{k=1}^{N} A_k \sin(k f n), \label{Eq:Set1}\\
\dot{\theta}^{\text{th}}_n + \dot{\theta}^{\text{calf}}_n &= \sum_{k=1}^{N} B_k \cos(k f n), \label{Eq:Set2}
\end{align}
where $N$ is the control sequence length. $f$ is the sampling frequency. $A_k$ and $B_k$ are the coefficients of the basis functions. 
\label{th:1}
\end{theorem}
\begin{proof}
For simplicity, $\dot{\theta}^{\text{calf}}_n$ is set to be zero for the proof of theorem 1. One can set all the initial values of the joint positions as 0, thereby 
\begin{align}
{\theta}^{\text{hip}}_n & =-\sum_{k=1}^{N} \frac{A_k}{k f} \cos (k f n), \\
{\theta}^{\text{th}}_n & =\sum_{k=1}^{N} \frac{B_k}{k f} \sin (k f n).
\end{align}
Then, one can deduce the foot-end angular velocities by Eq. (\ref{Eq:FK}), as given by
\begin{equation}
\omega^x_n = -\sum_{i=1}^{N} A_i \sin (i f n) \cdot \sin \left(\sum_{k=1}^{N} \frac{B_k}{k f} \sin (k f n)\right),
\end{equation}

\begin{equation}
\omega^y_n = -\sum_{i=1}^{N} A_i \sin (i f n) \cdot \cos \left(\sum_{k=1}^{N} \frac{B_k}{k f} \sin (k f n)\right),
\end{equation}

\begin{equation}
    \omega^z_n  = \sum_{k=1}^{N} B_k \cos(k f n),
\end{equation}
which yields to 
\begin{equation}
    \overline{\omega^y} = \overline{\omega^z} = 0, \label{Eq:mean}
\end{equation}
\begin{equation}
    \overline{\omega^x\omega^y} = \overline{\omega^x\omega^z} = \overline{\omega^y\omega^z} = 0. \label{Eq:orth}
\end{equation}

Furthermore, the auto-covariance matrix of theoretical foot-end angular velocity $\Sigma_{FF}$ can be calculated and simplified using Eq. (\ref{Eq:mean}), (\ref{Eq:orth}), and (\ref{Eq:FF}), which is expressed by

\begin{equation}
\Sigma_{F F} \approx\left(\begin{array}{ccc}
\overline{(\omega^x)^2}-(\bar{\omega^x})^2 & 0 & 0 \\
0 & \overline{(\omega^y)^2}-(\bar{\omega^y})^2 & 0 \\
0 & 0 & \overline{(\omega^z)^2}-(\bar{\omega^z})^2
\end{array}\right).
\end{equation}

Therefore, $\Sigma_{F F}$ is a diagonal matrix.

\end{proof}

After the proper selection of basis function sets for controlling angular velocity, $\Sigma_{F F}$ is ensured to be a diagonal matrix. To further minimize the condition number of $\Sigma_{F F}$, the coefficients $A_k$ and $B_k$ in Eq. (\ref{Eq:Set1}) and (\ref{Eq:Set2}) are optimized to make $\Sigma_{FF}$ a scalar matrix, so that the condition number is theoretically minimum. The objective function is defined as

\begin{equation}
\begin{gathered}
L(\mathbf{A},\mathbf{B})=\kappa\left(\Sigma_{F F}\right)+w^\text{hip} \Theta^\text{hip} +w^\text{th} \Theta^\text{th}+w^\text{calf} \Theta^\text{calf}, \\
\end{gathered}
\end{equation}
where $\mathbf{A}$ and $\mathbf{B}$ are the coefficients of the basis function sets. $w^\text{hip} \Theta^\text{hip} +w^\text{th} \Theta^\text{th}+w^\text{calf} \Theta^\text{calf}$ is the penalty for the limitation of the joints movement, which is given by
\begin{equation}
\begin{gathered}
\Theta_\text{idx}=\max \left(\theta^\text{idx}_{n}\right)-\min \left(\theta^\text{idx}_{n}\right), \\
w^\text{idx}=0, \text { if } \theta^\text{idx}_{n} \in\left[\theta^\text{idx}_{b}, \theta^\text{idx}_{u}\right], \\
\text{idx} = \text{hip}, \text{thigh}, \text{calf},
\end{gathered}
\end{equation}
where $[\theta^\text{idx}_{b}, \theta^\text{idx}_{u}]$ represents the joint position limits. Then gradient descent method is applied to this problem to optimize the coefficients $\mathbf{A}$ and $\mathbf{B}$ to minimize the condition number of $\Sigma_{FF}$. To conclude, our anti-noise referenced trajectory generation algorithm is illustrated in Algorithm \ref{alg:acmp}.

\subsection{RL-based Calibration Action Controller}

To enable the quadruped robot to autonomously perform the proposed active calibration actions, reinforcement learning is employed to ensure stability and smooth movement. The input observation, action space and reward design of the RL-based calibration action controller is presented as follows.


\subsubsection{Observation}

The input observation of the RL model includes 49 dimensions as listed in Table \ref{tab:OBS-Act}, which consists of the robot's proprioceptive information and high-level commands. 

\subsubsection{Action Space}

The action $\mathbf{a}_t$ generated by the policy networks consists of the target deviations from last positions for the twelve joints $\Delta \mathbf{q}_j^* = \{\Delta q_j\}$, where $q_j$ is the joint position of a certain joint.

\begin{table}[]
    \centering
    \caption{Observation and Action Term Summary}
    \renewcommand\arraystretch{1.3}
    \belowrulesep=0pt
    \aboverulesep=0pt
\resizebox{0.48\textwidth}{!}{
\begin{tabular}{ccc}
\toprule
\textbf{Observation Term Name} & \textbf{Defination} & \textbf{Dimension} \\
\cmidrule{1-3}
 Base Linear Velocity & $\mathbf{v}_B$ & $\mathbb{R}^3$  \\
 Base Angular Velocity & $\mathbf{\omega}_B$ & $\mathbb{R}^3$  \\
 Base Projected Gravity & $\mathbf{g}_B$ & $\mathbb{R}^3$  \\
 Joint Positions & $\mathbf{q}_j$ & $\mathbb{R}^{12}$ \\
 Joint Velocities & $\dot{\mathbf{q}}_j$ & $\mathbb{R}^{12}$ \\
 Previous Actions & $\mathbf{a}_{t-1}$ & $\mathbb{R}^{12}$ \\
 Tracking Commands & $\mathbf{\omega}_{B}^*$ & $\mathbb{R}^3$ \\
 Lifting Command & $c_{lift} \in \{0, 1, 2, 3, 4\}$ & $\mathbb{R}^1$ \\
\hline
\textbf{Action Term Name} \\
\hline
 Deviations from Previous Actions & $\Delta \mathbf{q}_j^*$ & $\mathbb{R}^{12}$  \\
\bottomrule
\end{tabular}}
\label{tab:OBS-Act}
\end{table}

\begin{table}[]
    \centering
    \caption{Reward Term Summary}
    \renewcommand\arraystretch{1.3}
    \belowrulesep=0pt
    \aboverulesep=0pt
\resizebox{0.48\textwidth}{!}{
\begin{tabular}{ccc}
\toprule
\textbf{Reward Term Name} & \textbf{Defination} & \textbf{Weight} \\
\cmidrule{1-3}
\hline
 \textbf{Calibration Task Reward} \\
\hline
 Lifting Foot & $\left\|\mathbf{F}_{\text{Calib-foot}}\right\|==0$ & $b_{lift}$ \\
 Tracking Angular Velocities & $\exp (-\frac{\left\|\mathbf{\omega}_{B}^*-\mathbf{\omega}_ {B}\right\|}{0.25})$ & 2.0  \\
 Slip Feet & $\sum_{i \in\{\text {remaining-feet}\}}\left\|\mathbf{v}_f\right\|$ & -0.1  \\
\hline
 \textbf{Normalized Reward} \\
\hline
 Termination & $\left\|\mathbf{F}_{\text{base}}\right\|>1.0$  & -100\\
 Collision Contacts & $\sum_{i \in\{\text {other-parts}\}}\left(\left\|\mathbf{F}_i\right\|>1.0\right)$ 
 & -1.0 \\
 Base Height & $\left\|h_B - h_B^*\right\|$ & -10.0 \\
 Base Orientation & $\left\| \Omega_B \right\|$ & -1 \\
 Base Linear Velocities & $\left\| \mathbf{v}_B \right\|$ & -1 \\
 Joint Accelerations & $\sum{\left\| \ddot{\mathbf{q}}_j \right\|}$ & -2.5e-7 \\
 Joint Torques & $\sum{\left\| \mathbf{\tau}_j \right\|}$ & -1e-5 \\
 Action Rate & $\sum{\left\| \mathbf{a}_{t} - \mathbf{a}_{t-1}\right\|} $ & -0.01 \\

\bottomrule
\end{tabular}}
\label{tab:reward}
\end{table}

\subsubsection{Reward}

Inspired by the single leg manipulation skills developed in \cite{Tripod}, our designed rewards for the calibration actions consist of two main components: 
\begin{itemize}
    \item Calibration task rewards \(R_t\) that encourage the robot to lift one leg, maintain a tripod posture, and execute the calibration action; and
    \item Normalized rewards \(R_n\) that promote safety, robustness, and smooth movement. 
\end{itemize}

The overall reward function is defined as: $R = R_t + R_n$. A detailed breakdown of the reward construction is provided in Table \ref{tab:reward}.
As for calibration task rewards $R_t$, the first reward component is a binary indicator \(b_{lift} \in \{-1, 1\}\), which signals whether the target calibration foot is off the ground. A positive reward is granted if the target foot remains in the air, whereas a penalty is imposed if the foot fails to lift for the calibration process. The second reward component ensures the stability of the robot’s base, aiming to maintain the validity of the Theorem 1 during calibration. The third reward component measures the slip of the remaining feet in contact with the ground, promoting better tripod stability and minimizing unwanted movements.

As for the normalized rewards $R_n$, the first part of which imposes a penalty for any contact between body parts other than the feet, ensuring safety during the calibration process. The second part penalizes deviations in base height, orientation, linear velocities, joint angular accelerations, torques, and action rates. These penalties are designed to enhance the smoothness and aesthetic quality of the robot’s movements.

\subsection{Implementation of A$^2$I-Calib on Quadruped Robots}
The RL-based calibration action controller is trained based on the Proximal Policy Optimization (PPO) algorithm\cite{PPO}. The actor-critic networks based on Multi-Layer Perceptrons (MLP) are trained in the IsaacGym simulation environment\cite{Gym}. Then, the trained controller is deployed on the Unitree Go2 quadruped robot. For better robustness, quadruped robot's initial pose and joint positions are pre-randomized, allowing for flexible execution of autonomous calibration at any time. Additionally, Gaussian noise is added to observations to achieve better generalization. The training process is divided into two phases: 1) standing in the tripod posture, and 2) execution of autonomous calibration actions. The model is trained with an episode of 20 seconds and commands are resampled every 10 seconds for both stages. Each stage takes 15 minutes on NVIDIA GeForce GTX 1660 Ti for 1000 iterations, exhibiting fine efficiency. Finally, the policy is deployed on Unitree Go2 quadruped robot with a joystick to send command. The reinforcement learning policy and PD controller both run at 50Hz.

\begin{figure}[]
\centering
{\includegraphics[width=0.45 \linewidth] {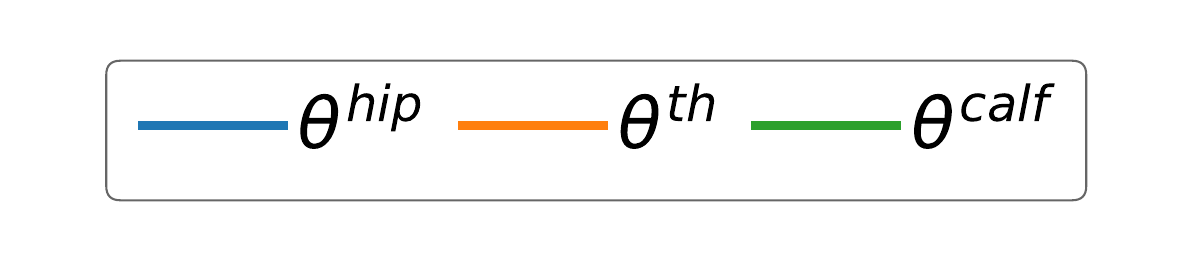}}\\%
\vspace{-0.4cm}
\subfloat[Joint positions in A$^2$I-Calib on Unitree Go2.]{\includegraphics[width=0.45 \linewidth]{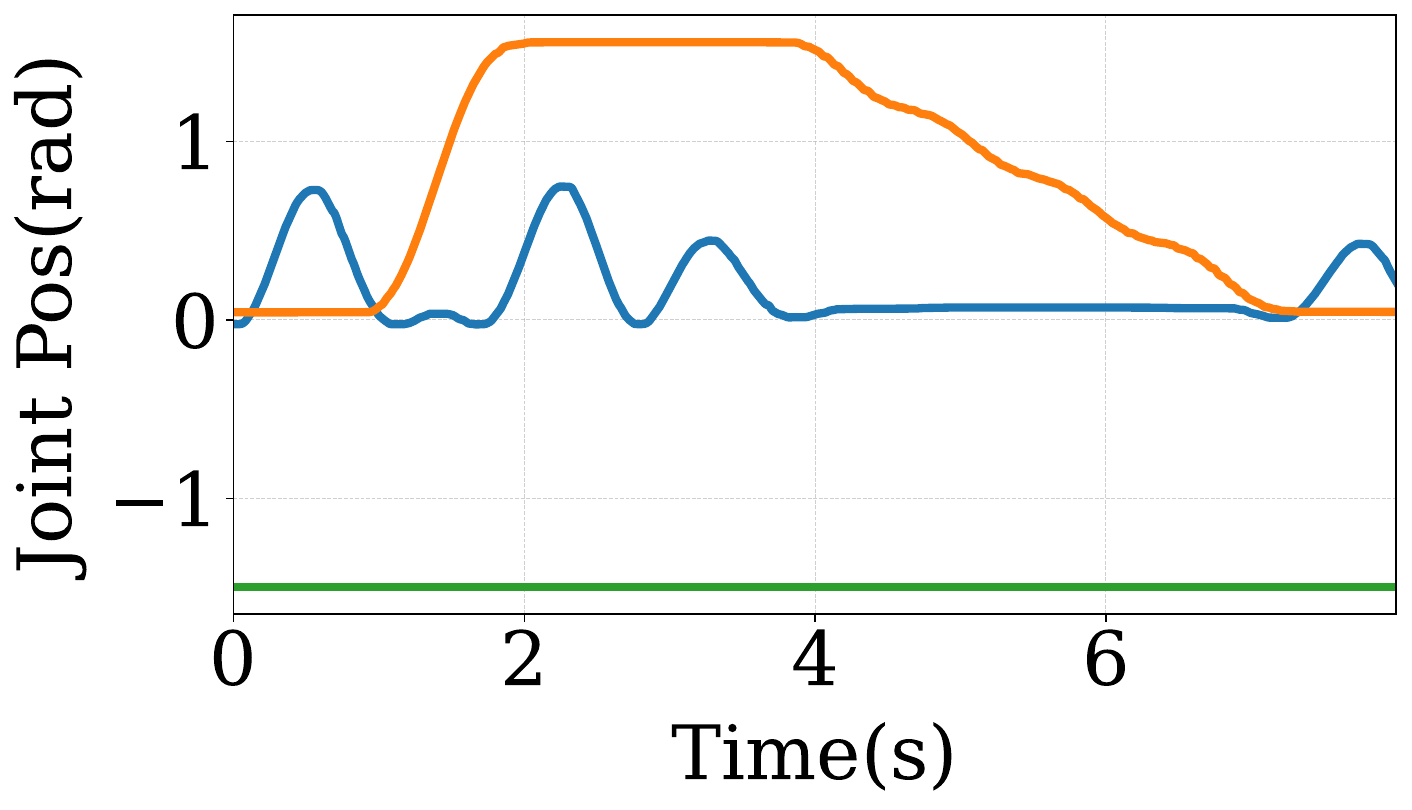}%
\label{fig_1_case}} \quad
\subfloat[Joint positions in walking gait on Unitree Go2.]{\includegraphics[width=0.45 \linewidth]{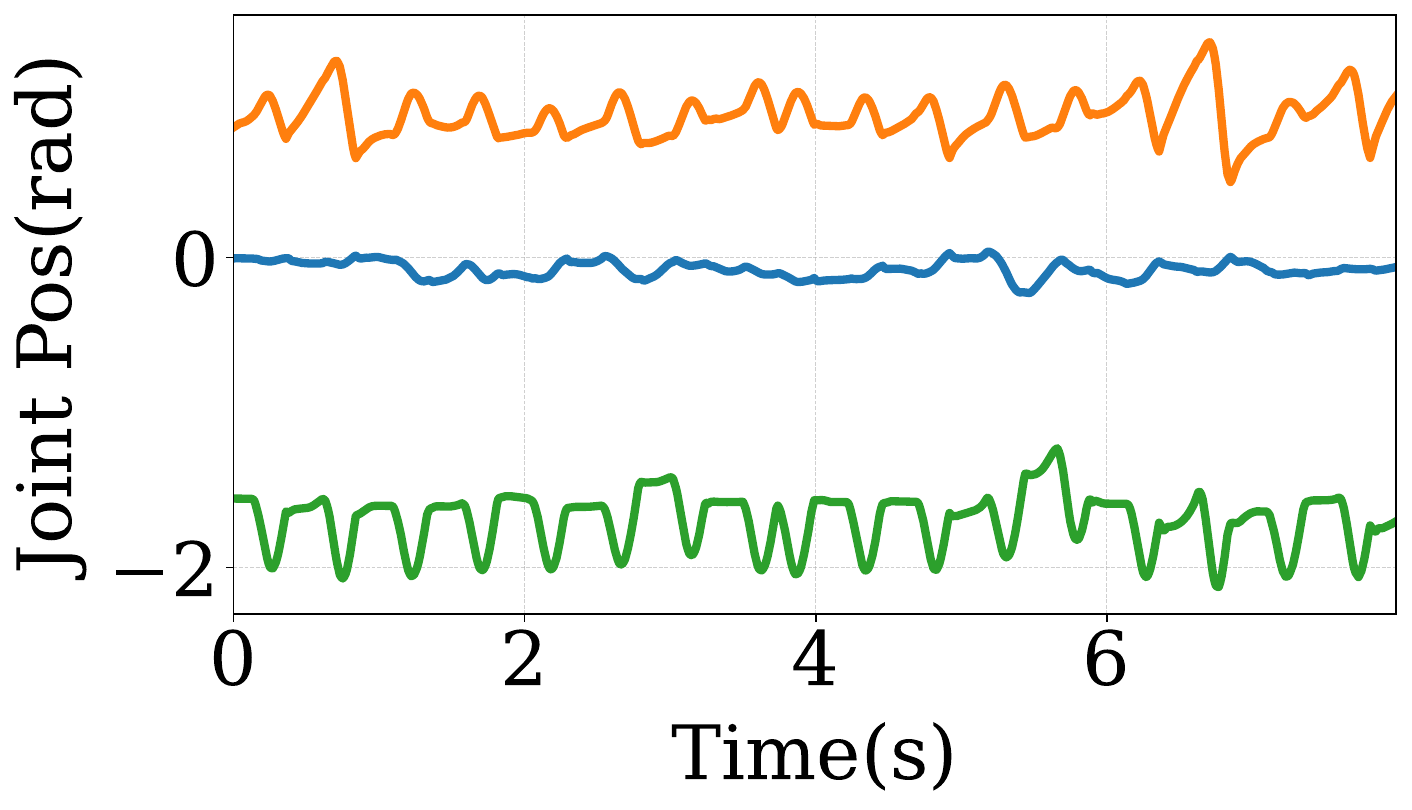}%
\label{fig_2_case}} \\
{\includegraphics[width=0.45 \linewidth] {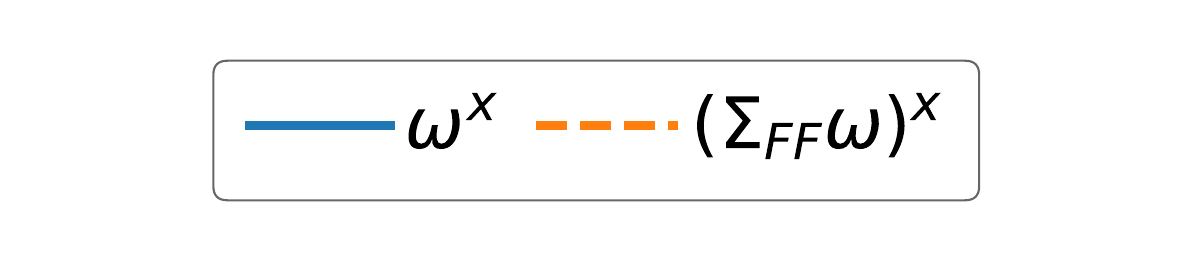}}\\%
\vspace{-0.4cm}
\subfloat[X-direction angular velocity affected by IMU noise in A$^2$I-Calib on Unitree Go2.]{\includegraphics[width=0.45 \linewidth]{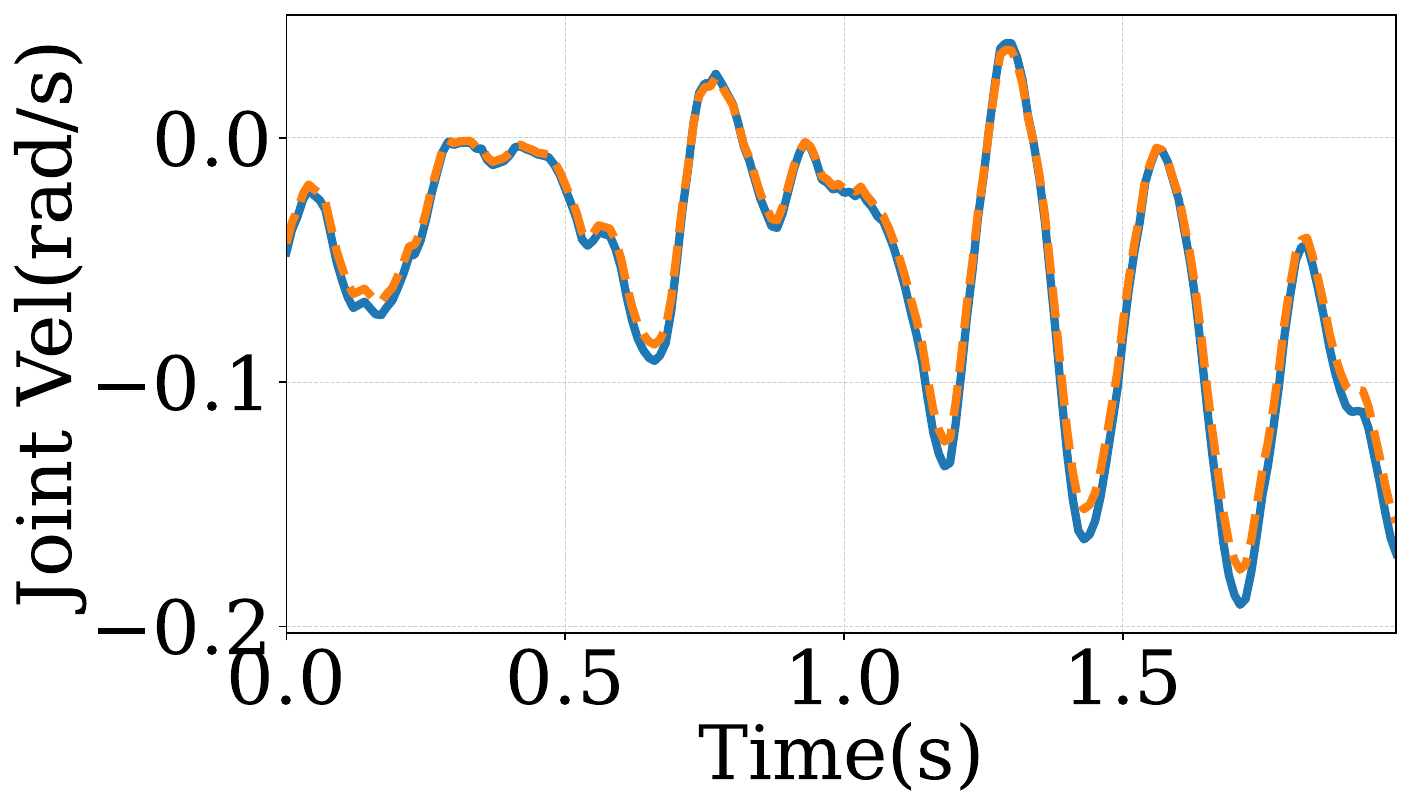}%
\label{fig_3_case}}  \quad
\subfloat[X-direction angular velocity affected by IMU noise in walking gait on Unitree Go2.]{\includegraphics[width=0.45 \linewidth]{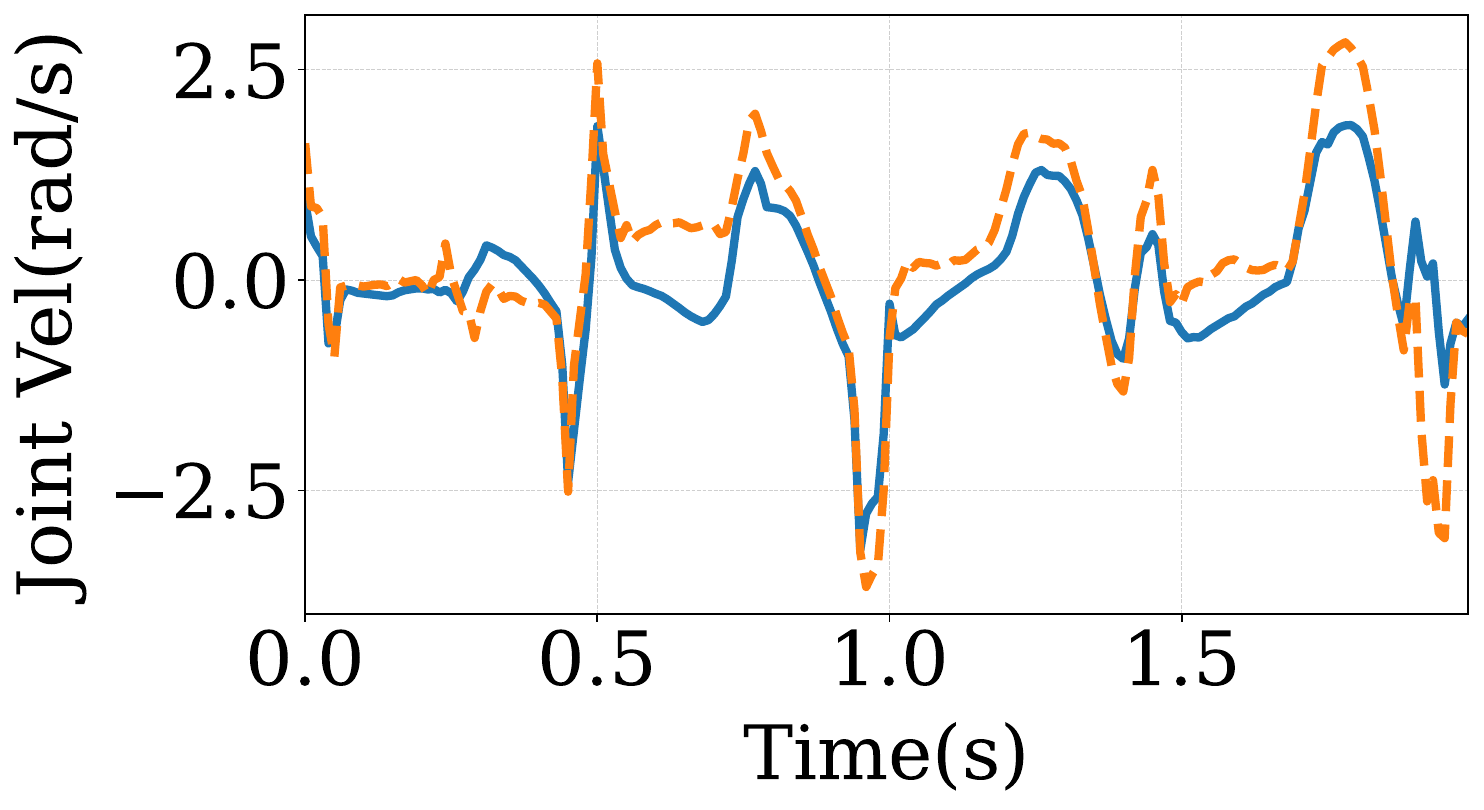}%
\label{fig_4_case}} \\
\subfloat[CN of A$^2$I-Calib VS. other gaits-based calibration in simulation.]{\includegraphics[width=1.0 \linewidth]{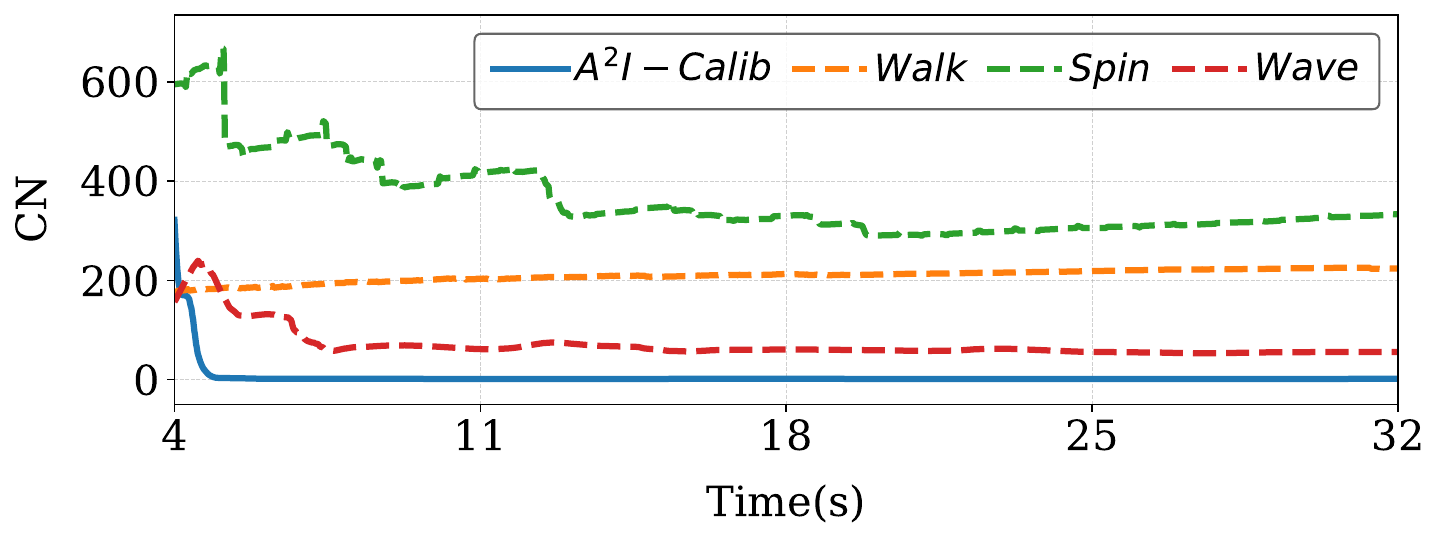}%
\label{fig_5_case}}
\caption{Simulation and real-world experiment results of joints positions, IMU's angular velocities and condition numbers.}
\label{fig:result}
\vspace{-0.2cm}
\end{figure}

\begin{table*}[]
    \centering
    \caption{Comparison of Condition Number, Correlation Coefficient and Rotation Error[${}^\circ$] with different calibration motion trajecrories.}
    \renewcommand\arraystretch{1.2}
    \belowrulesep=0pt
    \aboverulesep=0pt
\resizebox{0.9\textwidth}{!}{
\begin{tabular}{cccccccccccccc}
\toprule
\multirow{2}{*}{Node} & \multirow{2}{*}{Noise($^\circ/s/\sqrt{Hz}$)} & \multicolumn{3}{c}{Walk} & \multicolumn{3}{c}{Spin} & \multicolumn{3}{c}{Wave} & \multicolumn{3}{c}{\textbf{A$^2$I-Calib}} \\
\cmidrule(lr){3-5} \cmidrule(lr){6-8} \cmidrule(lr){9-11} \cmidrule(lr){12-14} & & CN & CC & RE & CN & CC & RE & CN & CC & RE & CN$\downarrow$ & CC$\uparrow$ & RE$\downarrow$ \\
\cmidrule{1-14}
\multirow{3}{*}{FL}
& 0.006 & 413.53 & 0.89 & 50.04 & 618.81 & 0.83 & 78.37 & 75.39 & 0.81 & 48.33 & \textbf{1.17} & \textbf{0.99} & \textbf{6.74} \\
& 0.03 & 217.54 & 0.93 & 29.28 & 442.97 & 0.83 & 71.20 & 72.52 & 0.80 & 56.50 & \textbf{1.23} & \textbf{0.99} & \textbf{7.35} \\
& 0.06 & 270.14 & 0.93 & 27.48 & 468.26 & 0.82 & 283.22 & 53.27 & 0.75 & 63.90 & \textbf{1.25} & \textbf{0.99} & \textbf{6.97} \\
\cmidrule{1-14}
\multirow{3}{*}{FR} 
& 0.006 & 345.46 & 0.92 & 62.53 & 99.70 & 0.89 & 26.27 & 52.70 & 0.67 & 238.61 & \textbf{1.25} & \textbf{0.99} & \textbf{1.03} \\
& 0.03 & 307.42 & 0.93 & 22.72 & 79.59 & 0.85 & 20.66 & 50.33 & 0.68 & 241.12 & \textbf{1.60} & \textbf{0.99} & \textbf{1.47} \\
& 0.06 & 326.18 & 0.93 & 38.03 & 88.65 & 0.82 & 41.41 & 55.44 & 0.65 & 236.92 & \textbf{1.52} & \textbf{0.99} & \textbf{1.41} \\
\cmidrule{1-14}
\multirow{3}{*}{RL} 
& 0.006 & 198.00 & 0.93 & 55.61 & 254.77 & 0.92 & 60.96 & 66.05 & 0.75 & 98.83 & \textbf{1.28} & \textbf{0.99} & \textbf{2.33} \\
& 0.03 & 182.10 & 0.93 & 38.24 & 193.00 & 0.92 & 8.89 & 52.83 & 0.71 & 76.60 & \textbf{1.51} & \textbf{0.99} & \textbf{2.02} \\
& 0.06 & 192.24 & 0.93 & 43.53 & 194.55 & 0.92 & 64.05 & 73.90 & 0.73 & 89.80 & \textbf{1.54} & \textbf{0.99} & \textbf{1.47} \\
\cmidrule{1-14}
\multirow{3}{*}{RR} 
& 0.006 & 125.23 & 0.94 & 74.99 & 82.91 & 0.91 & 75.46 & 23.08 & 0.80 & 16.74 & \textbf{1.21} & \textbf{0.99} & \textbf{2.04} \\
& 0.03 & 108.78 & 0.95 & 50.95 & 77.46 & 0.90 & 38.89 & 35.69 & 0.82 & 158.39 & \textbf{1.29} & \textbf{0.99} & \textbf{2.61}\\
& 0.06 & 115.97 & 0.95 & 62.94 & 75.37 & 0.90 & 91.07 & 26.09 & 0.81 & 55.20 & \textbf{1.37} & \textbf{0.99} & \textbf{1.44} \\

\bottomrule
\end{tabular}}
\label{tab:Gazebo_Calib}
\end{table*}

\begin{table}[]
    \centering
    \caption{Effect of Calibration Condition Number on Positioning Accuracy[m] on Gazebo platform.}
    \renewcommand\arraystretch{1.2}
    \belowrulesep=0pt
    \aboverulesep=0pt
\resizebox{0.48\textwidth}{!}{
\begin{tabular}{ccccccccccc}
\toprule
 \multirow{2}{*}{Noise} & \multicolumn{2}{c}{Walk} & \multicolumn{2}{c}{Spin} & \multicolumn{2}{c}{Wave} & \multicolumn{2}{c}{\textbf{A$^2$I-Calib}} \\
\cmidrule(lr){2-3} \cmidrule(lr){4-5} \cmidrule(lr){6-7} \cmidrule(lr){8-9} &  CN & APE & CN & APE & CN & APE  & CN$\downarrow$ & APE$\downarrow$  \\
\cmidrule{1-9}
0.006 & 270.56 & 0.26 &	264.05 & 0.89 & 54.31 &	0.26 & \textbf{1.23} &	\textbf{0.13} \\
0.03 & 203.96 & 0.19 & 198.26 & 0.19 & 52.84 & 0.55 & \textbf{1.41} &	\textbf{0.14}  \\ 
0.06 & 226.13 & 0.15 & 206.71 & 0.30 & 52.18 & 0.56 & \textbf{1.42} & \textbf{0.11} \\

\bottomrule
\end{tabular}}
\label{tab:MTR TTR}
\end{table}

\begin{table}[]
    \centering
    \caption{Condition Number, correlation coefficient and spatial calibration results on Unitree Go2.}
    \renewcommand\arraystretch{1.2}
    \belowrulesep=0pt
    \aboverulesep=0pt
\resizebox{1\linewidth}{!}{
\begin{tabular}{ccccccccc}
\toprule
\multirow{2}{*}{Node} & \multicolumn{3}{c}{Walk} & \multicolumn{3}{c}{\textbf{A$^2$I-Calib}} \\
\cmidrule(lr){2-4} \cmidrule(lr){5-7} & CN & CC & RPY & CN$\downarrow$ & CC$\uparrow$ & RPY \\
\cmidrule{1-7}
FL & 382.29 & 55.21 & (27, 10, 45) & \textbf{1.25} & \textbf{80.68} & \textbf{(104, 17, 21)} \\
FR & 541.94 & 54.52 & (162, -59, -140) & \textbf{1.32} & \textbf{80.67} & \textbf{(136, -11, 55)} \\
RL & 771.99 & 56.67 & (74, -35, -169) & \textbf{1.24} & \textbf{82.21} & \textbf{(56, -32, -129)} \\
RR & 978.50 & 56.99 & (-25, 29, -169) & \textbf{1.15} & \textbf{81.33} & \textbf{(92, -21, 115)} \\
\bottomrule
\end{tabular}}
\label{tab:MTR TTR}
\end{table}

\begin{figure}[]
\centering
\includegraphics[width=0.3\textwidth]{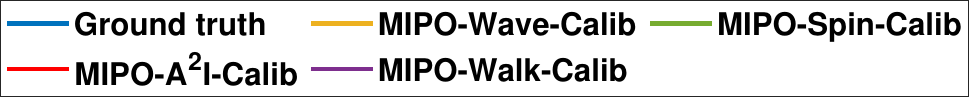} \\
\vspace{-0.2cm}
\subfloat[The estimated trajectories on the Gazebo sequence with noise level 0.06$^\circ/s/\sqrt{Hz}$.]{\includegraphics[width=0.42 \linewidth]{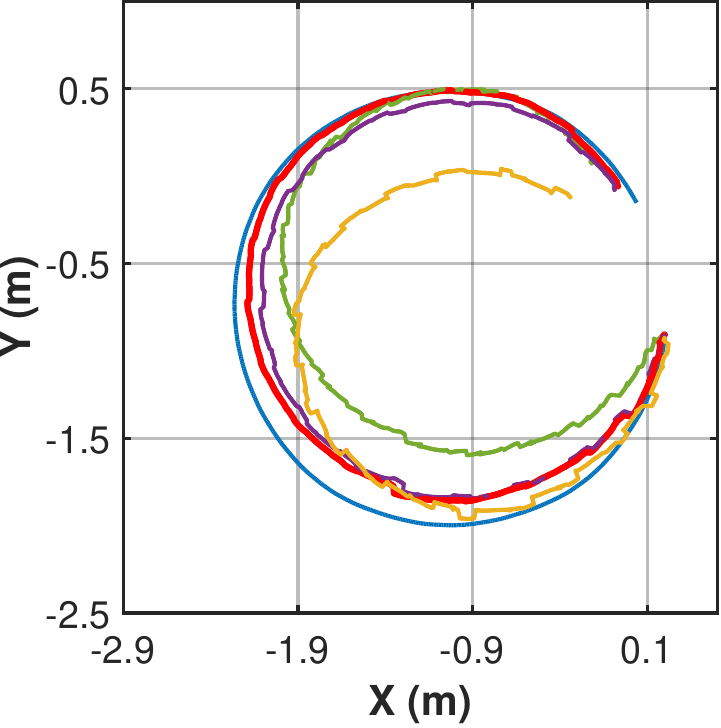}
\label{fig_1_case}}  \quad 
\subfloat[The estimated trajectories on the Gazebo sequence with noise level 0.006$^\circ/s/\sqrt{Hz}$.]{\includegraphics[width=0.42 \linewidth]{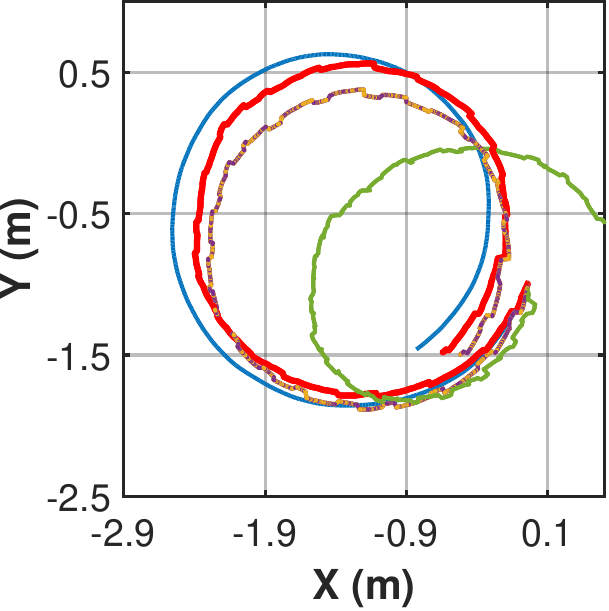}%
\label{fig_2_case}} \\

\caption{The estimated trajectories of MIPO with different calibration results in Gazebo.}
\label{fig:gazebo}
\vspace{-0.2cm}
\end{figure}

\section{EXPERIMENTS}
In this section, A$^2$I-Calib is evaluated against other naive calibration actions. First, we introduce the experimental setup, evaluation metrics, and platforms used for testing. Then, experiments are conducted on simulation and real-world legged robot platforms, focusing on two key aspects: 1) the noise sensitivity, and 2) the overall calibration accuracy of A$^2$I-Calib.
\subsection{Experimental Setup}
\subsubsection{Evaluation Metrics}
To evaluate the noise sensitivity and the overall calibration accuracy, we develop the following evaluation metrics:

\textit{Condition Number (CN)}: Condition number of the auto-covariance matrix for foot-end angular velocity is used to measure noise sensitivity of the calibration methods, as defined in Section III-C. Note that the ideal minimum CN equals to 1.

\textit{Correlation Coefficient (CC)}: Correlation coefficient of the foot IMU and joint encoder calculated foot-end angular velocity is used to assess the effectiveness of correlation affected by noise. Note that the ideal maximum CC equals to 1.

\textit{Rotation Error (RE)}: RE is defined to evaluate the rotation calibration error, as given by
\begin{equation}
    \text{RE}=\sqrt{\left|\gamma_x^*-\gamma_x^{\text {gt}}\right|^2+\left|\beta_y^*-\beta_y^{\text {gt}}\right|^2+\left|\alpha_z^*-\alpha_z^{\text {gt}}\right|^2},
\end{equation}
where $\gamma_x^*$, $\beta_y^*$ and $\alpha_z^*$ are the calibration results for the Euler angles from foot end to foot IMU. $\gamma_x^{\text {gt}}$, $\beta_y^{\text {gt}}$ and $\alpha_z^{\text{gt}}$ are the corresponding ground truth angles.


\textit{Absolute Positioning Error (APE)}: APE of multi-IMU proprioceptive odometry (MIPO) on the same sequence with different calibration results is used to reflect the impact of calibration accuracy on the overall odometry performance. The root mean square (RMS) of APE is adopted in this paper for comparison.

Note that during the experiment, the temporal calibration of all trajectory sequences is precise to ms level, which have limited impact on the multi-IMU system. As a result, the time offset calibration error is not listed as a core metric in this paper.

\subsubsection{Development of Experimental Platforms}
To demonstrate the preciseness and anti-noise capability of A$^2$I-Calib, experiments are conducted on simulation platform Gazebo with strict ground truth and real-world platform Unitree Go2 with multi-IMU settings.

\subsubsection{Compared Methods}
A$^2$I-Calib is compared against other existing gait trajectories with open-loop foot IMU calibration, including walking, spinning and naive single foot waving.

\begin{figure}
    \centering
\subfloat[Unitree Go2 Platform]{\includegraphics[width=0.45 \linewidth]{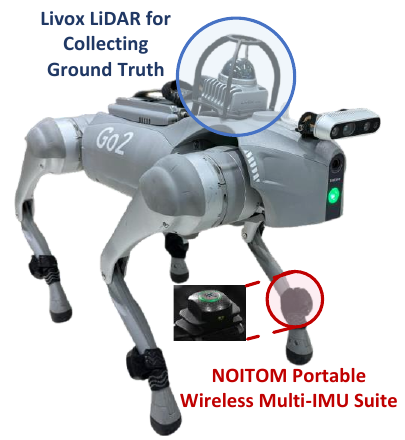}%
\label{fig_1_case}} \quad
\subfloat[The estimated trajectories of MIPO in the real world.]{\includegraphics[width=0.45 \linewidth]{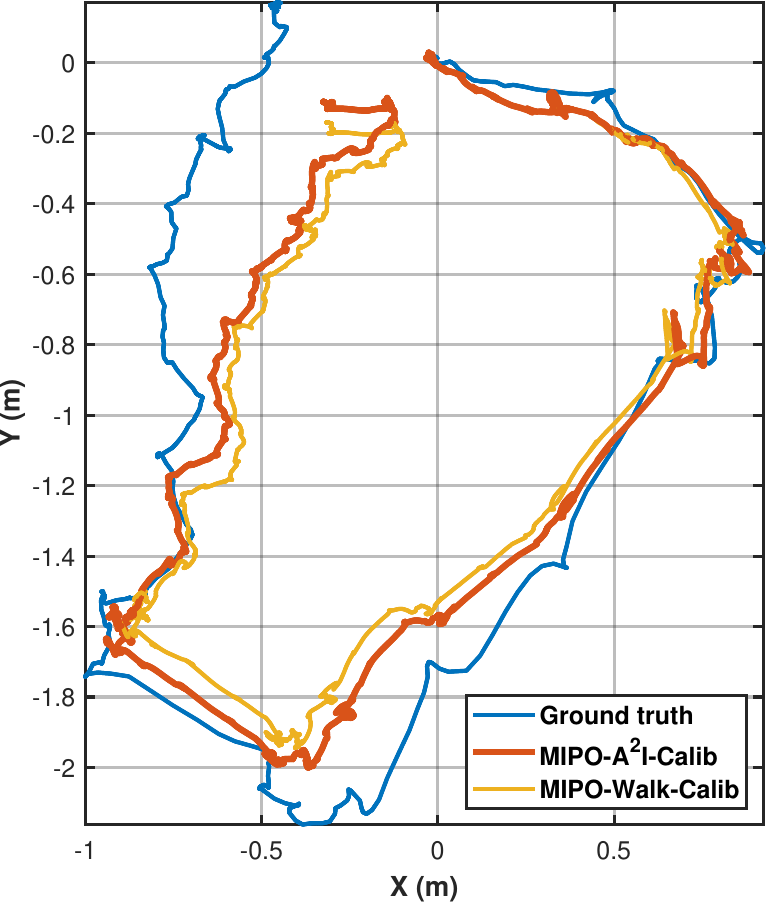}%
\label{fig_2_case}} \quad
    \caption{Experiments on the real-world quadruped robot platform, Unitree Go2. \textbf{(a):} The real-world multi-IMU system based on the Unitree Go2, equipped with Livox LiDAR for odometry ground truth collection and a NOITOM portable wireless multi-IMU suite. \textbf{(b):} The estimated trajectories of MIPO with different calibration results in real-world scenarios. The RMS of APE for MIPO-A$^2$I-Calib is 0.23 m, while for MIPO-WALK-Calib, it is 0.27 m, indicating a 15\% improvement in accuracy for MIPO-A$^2$I-Calib.}
    \label{fig:platform}
    \vspace{-0.5cm}
\end{figure}
\subsection{Experimental Results and Discussions}
Selective simulation and real-world experiment results of joints positions, IMU's angular velocities and condition numbers are shown in Fig. \ref{fig:result}. Detailed discussions as presented as follows.
\subsubsection{Simulation}
The simulation experiments evaluate the calibration results of four different gaits, including walking, spinning, leg waving, and proposed A$^2$I-Calib action, under three IMU noise levels: 0.006, 0.03, and 0.06$^\circ/s/\sqrt{Hz}$. The calibration results are then applied to assess MIPO performance. A$^2$I-Calib action achieves significant reductions in noise sensitivity compared to the other gaits' trajectories. With lower noise sensitivity A$^2$I-Calib thus improves rotational calibration accuracy. Furthermore, the APE of MIPO is reduced by 36.5\%, 72.4\%, and 72.2\% on sequences with different levels of IMU noise, respectively. Part of the estimated trajectories are shown in Fig. \ref{fig:gazebo} for comparison.

\subsubsection{Real-world Quadruped Robots}
Real-world experiments compare the calibration noise sensitivity of two different gaits, including walking and proposed A$^2$I-Calib action using Noitom IMUs. Fast-LIO\cite{Fast-LIO} serves as the ground truth to test the performance of MIPO utilizing different calibration results. A$^2$I-Calib action shows a reduction in noise sensitivity in comparison to walking gait, thus resulting in the reduction of the APE of MIPO by 15\%, as illustrated in Fig. \ref{fig:platform}.


\section{CONCLUSIONS}
This paper presents A$^2$I-Calib, an active noise-resistant calibration framework that enables fully autonomous spatial-temporal calibration of arbitrary foot-mounted IMUs. To minimize the noise sensitivity during the calibration process, an anti-noise trajectory generation module is designed. It employs a proposed basis function selection theorem to optimize condition numbers in correlation analysis. Then, a RL-based calibration action controller is proposed to ensure the robust execution of calibration-specific motions. Comprehensive validations in simulated and physical environments demonstrate the capability of A$^2$I-Calib to achieve high-precision calibration with wide-range IMU noise levels. Future works will be focused on the integration of other sensors in the active calibration framework to achieve fully autonomous calibration systems for multi-node legged robots.





\normalem
\bibliographystyle{IEEEtran}
\bibliography{ref}

\end{document}